\documentclass[letterpaper, 10pt, conference]{ieeeconf}      % Use this line for a4
\IEEEoverridecommandlockouts                              % This command is only
\usepackage[utf8]{inputenc}
\usepackage[T1]{fontenc}
\usepackage{graphicx}      % include this line if your document contains figures
\usepackage{amsmath}
\usepackage{amssymb}
\usepackage{esint}
\usepackage{tikz}
\usepackage{multirow}
\usepackage{booktabs}
\usepackage{nicematrix} 
\usepackage{comment}
\usepackage{siunitx}
\makeatletter               %To use hyperref with IEEEconf temoplate, put last
\let\NAT@parse\undefined    %To use hyperref with IEEEconf temoplate, put last
\makeatother   %To use hyperref with IEEEconf temoplate, put last
\usepackage{hyperref}   %To use hyperref with IEEEconf temoplate, put last

\newtheorem{assumption}{Assumption}
\newtheorem{definition}{Definition}
\newtheorem{remark}{Remark}

\newtheorem{lemma}{Lemma}

\title{\LARGE \bf
Fast Collision Probability Estimation for Automated Driving using Multi-circular Shape Approximations}

\author{Leon Tolksdorf$^{1, 2}$, Christian Birkner$^{2}$, Arturo Tejada$^{1, 3}$, and Nathan van de Wouw$^{1}$ % <-this % stops a space
\thanks{$^{1}$Department of Dynamics and Control, Eindhoven University of Technology, Eindhoven, The Netherlands, e-mail:
        {\tt\small \{l.t.tolksdorf, a.tejada.ruiz, n.v.d.wouw\}@tue.nl}}%
\thanks{$^{2}$CARISSMA Institute of Safety in Future Mobility, Technische Hochschule Ingolstadt, Ingolstadt, Germany, e-mail:
        {\tt\small \{leon.tolksdorf, christian.birkner\}@thi.de}}%
\thanks{$^{3}$TNO, Integrated Vehicle Safety, Helmond, The Netherlands, e-mail:
        {\tt\small arturo.tejadaruiz@tno.nl}}%
}

\newcommand\copyrighttext{%
  \footnotesize \textcopyright 2024 IEEE.  Personal use of this material is permitted.  Permission from IEEE must be obtained for all other uses, in any current or future media, including reprinting/republishing this material for advertising or promotional purposes, creating new collective works, for resale or redistribution to servers or lists, or reuse of any copyrighted component of this work in other works.}
\newcommand\copyrightnotice{%
\begin{tikzpicture}[remember picture,overlay]
\node[anchor=south,yshift=10pt] at (current page.south) {\fbox{\parbox{\dimexpr\textwidth-\fboxsep-\fboxrule\relax}{\copyrighttext}}};
\end{tikzpicture}%
}

\begin{document}

\maketitle
\thispagestyle{empty}
\pagestyle{empty}
\copyrightnotice
\begin{abstract}

Many state-of-the-art methods for safety assessment and motion planning for automated driving require estimation of the probability of collision (POC). 
To estimate the POC, a shape approximation of the colliding actors and probability density functions of the associated uncertain kinematic variables are required. Even with such information available, the derivation of the POC is in general, i.e., for any shape and density, only possible with Monte Carlo sampling (MCS). Random sampling of the POC, however, is challenging as computational resources are limited in real-world applications. We present expressions for the POC in the presence of Gaussian uncertainties, based on multi-circular shape approximations. In addition, we show that the proposed approach is computationally more efficient than MCS. Lastly, we provide a method for upper and lower bounding the estimation error for the POC induced by the used shape approximations.

\end{abstract}
%%%%%%%%%%%%%%%%%%%%%%%%%%%%%%%%%%%%%%%%%%%%%%%%%%%%%%%%%%%%%%%%%%%%%%%%%%%%%%%%
\section{Introduction}\label{sec_introduction}

In automated driving (AD), stochastic algorithms accounting for uncertainties in measurement and estimation, as well as the uncertainty associated with predicting future traffic scenarios are thriving as it is understood that deterministic algorithms cannot account for such uncertainty (see, e.g., \cite{schwarting2018planning,McAllister.82017}). Besides academia, safety standardization \cite{ISO21448} also recognizes the need to consider uncertainties due to, e.g., sensor and model limitations. \\
The handling of uncertainty in motion planning and decision-making modules of automated vehicles depends on the specific algorithms. Prominently, the probability of collision (POC) (see, e.g., \cite{schwarting2017safe, goulet2022probabilistic, schreier2016integrated, althoff2009model}) or risk (see, e.g., \cite{tolksdorf2023risk, hruschka2019uncertainty, nyberg2021risk}), as risk is in many cases defined on the basis of the POC, has to be estimated along the considered motion plan. Subsequently, a safe set can be constructed, composed of states for which a risk or POC constraint is satisfied (see, e.g., \cite{brudigam2021stochastic, muller2022motion, batkovic2022safe}). 
The commonality for all those motion planning approaches is that they must assess for each time step along a motion plan whether a certain POC is exceeded. In practice, this has to be done for thousands of trajectories during one planning cycle. Thus, computational efficiency becomes an important, practical requirement for the application of such algorithms. \\
For colliding objects of arbitrary shapes associated with probability density functions (PDFs), the Monte Carlo sampling (MCS) approach presents the general solution to compute the POC and has found wide adoption in literature \cite{lambert2008collision}. Computing the POC with MCS is, however, computationally expensive. Also, MCS is not guaranteed not to underestimate the POC for finite amount of samples, which is undesirable from a safety perspective. Thus, alternative approximations for more specific cases have been introduced. For example, \cite{du2011probabilistic} improves the MCS approach by formulating a joint object's (i.e., both vehicles) distribution for Gaussian uncertainties. The authors of \cite{philipp2019analytic} consider Gaussian distributed uncertainties for rectangular shapes but assume deterministic heading angles for both objects. The same assumption is also made by \cite{altendorfer2021new}, to derive the POC for a time interval. In \cite{patil2012estimating}, the vehicle's shape is assumed to be point-like and the POC is also estimated for a time interval. \\% Rectangular shapes are also used in \cite{tejada2019}, however the set of collision states could only be upper bounded and the PDF is assumed to be uniform.
Here, we introduce a computationally efficient alternative to approximate the POC based on multi-circular shape approximations. That is, by using  overlapping circles to cover a vehicle's shape as presented in \cite{ziegler2010fast}. This approximation has found widespread use in motion planning (see, e.g., \cite{gutjahr2016lateral, manzinger2020using, werling2012optimal}) as an efficient way to checking for collisions\footnote{It only requires verifying that the distance between two circle centers is less or equal than the sum of both radii.}. On the other hand, to the best of the author's knowledge, the estimation of the POC using the multi-circular shape approximation has not yet been reported in the literature.\\ 
This article presents a method to derive the POC for a multi-circular ego vehicle approximation colliding with a circular object approximation. In particular, we focus on Gaussian distributed positions of the object vehicle, whereas the ego vehicle is of known position and orientation. The main contributions of this paper are as follows: 1) We derive POC approximations in local, global, and polar coordinates for circle-to-circle collisions and compare such alternatives with respect to computational effort. We find that our proposed method in local coordinates is most efficient. 2) We present an algorithm computing the POC for generalized multi-circle approximations of the ego vehicle colliding with a circular object and prove that it retains a well-defined probability when given geometric conditions are satisfied. 3) We present a method to upper and lower bound the error introduced by the multi-circular approximation and demonstrate the evolving error in a representative automated driving scenario. \\
The remainder of this article is organized as follows. Section \ref{sec_problem} presents basic definitions and the assumptions on the basis of which the problem of approximating the POC with multiple circles is then formulated. We solve the basic case of circle-to-circle collisions in Section \ref{sec_circle_to_circle}, where we also compare various approaches. The presented results are then extended to the multi-circle-to-circle cases in Section \ref{sec_multi_circle}. Next, we present a method to bound the POC approximation error in Section \ref{sec_error} and provide a representative example in Section \ref{sec_example}. Lastly, we discuss the results, provide a conclusion and outline future work in the respective Sections \ref{sec_discussion} - \ref{sec_conclusion}.

\section{Problem Formulation}\label{sec_problem}
Consider an arbitrary traffic scene with an automated (or ego) vehicle and one dynamic object. Let each vehicle or object be characterized by a configuration $\boldsymbol{y} := (\boldsymbol{q}, \theta) \in \mathcal{C}$ where $\mathcal{C}$, the configuration space, $\boldsymbol{q} := (c_1, c_2) \in \mathbb{R}^2$ denotes the set of possible positions of the geometric center of a vehicle, and $\theta \in [0, 2\pi)$ the vehicle's heading angle. All variables associated with the ego vehicle and the object will be identified, respectively, with $e$ and $o$ subscripts. All random variables are defined over the same probability space $(\Omega, \mathcal{F}, \mathbb{P})$, where $\Omega$ is the sample space, $\mathcal{F}$ is a sigma-algebra over $\Omega$, and $\mathbb{P}$ is a probability measure over $\mathcal{F}$. The sets in $\mathcal{F}$, also known as events, are denoted by $\mathcal{A}_n$ with $n \in \mathbb{N}$. In the sequel, the following definitions will be needed.

\begin{definition}\label{def011}(Probability Measure)
A mapping $\mathbb{P}: \mathcal{F} \rightarrow [0, 1]$ is a probability measure if:
    \begin{itemize}
        \item[(i)] $\mathbb{P}\{\Omega\} = 1$ and $\mathbb{P}\{\emptyset\} = 0$,
        \item[(ii)] $\mathbb{P}\{\bigcup_{n =1}^{\infty} \mathcal{A}_n\} = \sum_{n =1}^{\infty}\mathbb{P}\{ \mathcal{A}_n\}$, if $\mathcal{A}_n$ are pairwise disjoint (i.e., $\mathcal{A}_n \cap \mathcal{A}_j = \emptyset$, $n\!\neq\!j$ and $\mathcal{A}_n,\mathcal{A}_j \in \mathcal{F}$).
    \end{itemize}
\end{definition}
\begin{definition}\label{def_random_vector}(Random Element)
    Let $(\Omega, \mathcal{F})$ and $(\tilde{\Omega}, \tilde{\mathcal{F}})$ be measurable  spaces. A map
    $\boldsymbol{x} : \Omega \rightarrow \tilde{\Omega}$ is called a random vector if $\boldsymbol{x}^{-1}[\tilde{\mathcal{A}}_n] \in \mathcal{F}$ for all $\tilde{\mathcal{A}}_n \in \tilde{\mathcal{F}}$.
\end{definition}

\subsection{Collision Probability for Arbitrary Shapes}
The ego vehicle must estimate the POC with other objects at all times along its motion plan. Thus, the ego vehicle should be equipped with an algorithm that predicts future (uncertain) configurations of these objects. We assume the following information to be available.
\begin{assumption}\label{ass1}(Information) 
The object configuration $\boldsymbol{y}_{o,k} \in \mathcal{C}$ at time $k$ is measured (either by direct measurement or estimation) together with an associated multivariate Gaussian probability density $p_{\boldsymbol{y}, k}$ which represents the uncertainty of the objects' three independent configuration variables. The configurations associated with the ego vehicle are assumed to be measurable at all times without uncertainty. 
\end{assumption}
In general, the uncertainties associated with the configurations of the ego vehicle are less significant than those of the object, justifying Assumption \ref{ass1}. Because the POC is estimated for all times using only the information available at time $k$, i.e., the object configuration and its associated PDF at time $k$, the sub-index $_k$ will be omitted in the sequel to simplify the  notation. \\
A collision occurs if the space occupied by the ego vehicle (i.e., the ego footprint) and that of the object intersect. Since the physical shapes of the ego and the object are bounded, it will be assumed that their footprints constitute compact regions over $\mathbb{R}^2$ denoted, respectively, by $\mathcal{S}_e(\boldsymbol{y}_{e})$ and $\mathcal{S}_o(\boldsymbol{y}_{o})$. Since by assumption $\boldsymbol{y}_{e}$ is deterministic, the POC is induced only by the random vector $\boldsymbol{y}_{o}$. More specifically, if $\tilde{\mathcal{A}}_{coll} := \{\boldsymbol{y} \in \mathcal{C} \mid \mathcal{S}_e(\boldsymbol{y}_{e}) \cap\mathcal{S}_o(\boldsymbol{y}) \neq \emptyset\}$ 
the POC is given by
\begin{equation}\label{general}
\textup{POC}\triangleq \mathbb{P}\{\boldsymbol{y}_{o} \in\tilde{\mathcal{A}}_{coll}\} = \int\displaylimits_{ \tilde{\mathcal{A}}_{coll}} p_{\boldsymbol{y}_o}(\boldsymbol{y}) \text{d} \boldsymbol{y}.
\end{equation}

\subsection{Multi-Circle Shape Approximations}
\begin{figure}[t!]
\begin{center}
\includegraphics[width=8.4cm]{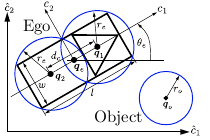}  % The printed column width is 8.4 cm.
\setlength{\belowcaptionskip}{-10pt}
\caption{Depiction of the problem statement. The global coordinate axis are represented by $\hat{c}_{1},\hat{c}_{2}$, the vehicle coordinate system by $c_1, c_2$.}
\label{fig_problem_statement}
\end{center}
\end{figure}

We approximate the ego vehicle with multiple circles, as proposed in \cite{ziegler2010fast}, and the object with a single circle, allowing us to neglect the objects' heading angle, reducing the later derivations by one dimension. In the sequel, $\mathcal{B}[\boldsymbol{q};r]$ will denote a closed circle with center $\boldsymbol{q} = (c_{1}, c_{2})^T$ and radius $r \in \mathbb{R}_{>0}$. That is, 
\begin{equation}\label{shape}
\mathcal{B}[\boldsymbol{q};r] := \{ \boldsymbol{x} \in \mathbb{R}^2 \mid \| \boldsymbol{q} - \boldsymbol{x}\| \leq r \}.
\end{equation}
While the object is assumed to be a single circle with radius $r_o$, suppose the ego vehicle is of rectangular shape with length $l$ and width $w$, $l > w$. Further, suppose the ego should be covered by $N_c$ overlapping, closed circles with centers $\boldsymbol{q}_i$, $i=1,\dots,N_c$, and equal radii $r_e$. This gives the over-approximation: 
\begin{equation}\label{eq_shape}
\mathcal{S}_e(\boldsymbol{y}_e) \subset \bigcup_{i=1}^{N_c}\mathcal{B}[\boldsymbol{q}_i;r_e].
\end{equation}
It can be shown that the smallest radius $r_e$ needed to indeed cover the entire rectangle and the related distance between consecutive centers $d_c=\|\boldsymbol{q}_{i+1}-\boldsymbol{q}_i\|$, $i=1,\dots, N_c-1$, aligned on one axis, is given by
\begin{equation}\label{eq_placing}
    r_e = \sqrt{\left(\frac{l}{2N_c }\right)^2 + \frac{w^2}{4}}, \; \; \; d_c = 2\sqrt{r_e^2 -\frac{w^2}{4}}.
\end{equation}

 To summarize, the problem setup is illustrated in Figure \ref{fig_problem_statement} for a two-circle approximation of the ego vehicle, and the problem is stated next.\\
\textit{Problem}: Given Assumption \ref{ass1} and the circles covering the ego vehicle (\ref{shape}) - (\ref{eq_placing}), derive an over-approximation of (\ref{general}).

\section{Circle-to-Circle Collisions}\label{sec_circle_to_circle}
To understand the derivation for multi-circle-to-circle collisions it is required to derive circle-to-circle collisions (i.e., there is one ego vehicle circle and one object circle) first as the extension to multi-circles follows from this case and is provided in Section \ref{sec_multi_circle}. By Assumption \ref{ass1}, the ego vehicle's information about the object is uncertain. Moreover, using the single circle shape approximation of the object, we can formulate the uncertainty in terms of a Gaussian distribution that depends only on the position $\boldsymbol{q}_o = (c_{1}, c_{2})$ in the ego vehicle frame and therefore restricts us to a bivariate Gaussian, i.e.,
\begin{equation}\label{gauss}
\begin{split}
    &p_{\boldsymbol{q}_o}(c_1, c_2) = \\
    &\frac{1}{2 \pi \sigma_{c_1}\sigma_{c_2}} \exp{\left(-\frac{1}{2} \left( \frac{(c_1 - \mu_{c_1})^2}{\sigma_{c_1}^2} +  \frac{(c_2 - \mu_{c_2})^2}{\sigma_{c_2}^2} \right) \right)}.
\end{split}
\end{equation}
Here, $\sigma_{c_1}, \sigma_{c_2}$ denote the respective standard deviations and $\mu_{c_1}, \mu_{c_2}$ the respective mean values of the object's two Cartesian positions $c_1$ and $c_2$ with respect to the ego vehicle frame. Equation (\ref{gauss}) represents a non-zero-mean, anisotropic bivariate normal distribution. For the circle-to-circle collision, we obtain the collision event as $\tilde{\mathcal{A}}^B := \{\boldsymbol{q} \in \mathbb{R}^2 \mid \mathcal{B}[\boldsymbol{q}_e, r_e] \cap \mathcal{B}[\boldsymbol{q}, r_o] \neq \emptyset\}$. In the following, we introduce four different approaches to compute the POC using (\ref{gauss}). Besides the baseline MCS approach, the derivations and differences among the other approaches are helpful results for the practitioner. 

\subsection{Baseline, Approach 1: Monte Carlo Sampling}\label{sec_MCS}
Our baseline approach is MCS, since it has found widespread adoption within the literature (see, e.g., \cite{lambert2008collision}). For the MCS solution, we introduce the collision indicator function 
\begin{equation}\label{eq_indicator}
    I_C(\boldsymbol{q}_{e}, \boldsymbol{q}_{o}) = 
  \begin{cases}
    1       \quad \text{if } d \leq r_e + r_o, \\
    0   \quad \text{otherwise,}
  \end{cases}
\end{equation}
where $d = \| \boldsymbol{q_e} -\boldsymbol{q_o} \|$ in (\ref{eq_indicator}) represents the Euclidean distance between the centers of both actor's circles. We approximate the POC by the law of large numbers, where $\boldsymbol{q}_{o,j}$ is one of $N_J$ samples drawn from the density $p_{\boldsymbol{q}_o}$ in (\ref{gauss}). Thus, we approximate the POC by
\begin{equation}\label{eq_MCS}
\begin{split}
\frac{1}{N_J}\sum_{j = 1}^{N_J} I_C(\boldsymbol{q}_{e}, \boldsymbol{q}_{o,j}) \approx \mathbb{P}\{\boldsymbol{q}_o \in \tilde{\mathcal{A}}^B\}.
\end{split}
\end{equation}

\subsection{Proposed Approach 2: Local Coordinates}\label{sec_local_coordinates}
\begin{figure}
\begin{center}
\includegraphics[width=8.4cm]{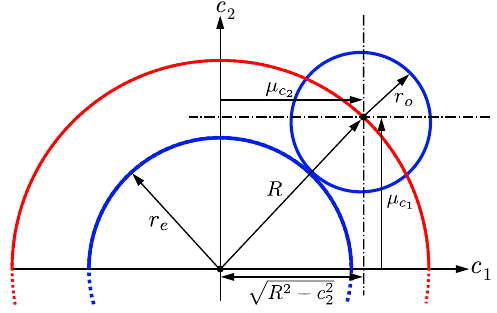}  % The printed column width is 8.4 cm.
\caption{Determination of the integral bounds in local coordinates, where $c_2 \in [-R, R]$ and $c_1 \in \left[-\sqrt{R^2-c_2^2}, \sqrt{R^2-c_2^2}\right]$.}
\label{local_circles}
\end{center}
\end{figure}
We propose to solve the problem in the ego vehicle coordinate frame, such that (\ref{gauss}) directly applies; see Figure \ref{local_circles}. The red circle represents the collision radius $R = r_e + r_o$, visualizing that all object positions within the compact red circle lead to a collision. Thus, the red circle provides the integration region in (\ref{general}). We can explicitly express $\tilde{\mathcal{A}}^B := \{\boldsymbol{q} \in \mathbb{R}^2 \mid c_1^2 + c_2^2 \leq R^2 \}$. With the transformation and integral bounds, we can apply Equation (\ref{gauss}) to (\ref{general}), yielding:
\begin{equation}\label{own_1}
    \begin{split}
     &\mathbb{P}\{\boldsymbol{q}_o \in \tilde{\mathcal{A}}^B\} = \frac{1}{2 \pi \sigma_{c_1}\sigma_{c_2}} \int_{-R}^{R}\int_{-\sqrt{R^2 - c_2^2}}^{\sqrt{R^2 - c_2^2}} \\
     &\exp{\left(-\frac{1}{2} \left( \frac{(c_1 - \mu_{c_1})^2}{\sigma_{c_1}^2} +  \frac{(c_2 - \mu_{c_2})^2}{\sigma_{c_2}^2} \right) \right)} \text{d}c_1 \text{d}c_2.
     \end{split}
\end{equation}
By Assumption \ref{ass1} both random variables $c_1, c_2$ are independent and by considering $\exp{(x+y)} = \exp{(x)}\exp{(y)}$, we can evaluate the inner and outer integral separately, where each integral involves the one-dimensional normal distribution. Knowing the cumulative distribution function of such, we can solve the inner integral and obtain
\begin{equation}\label{own_2}
\begin{split}
    & \mathbb{P}\{\boldsymbol{q}_o \in \tilde{\mathcal{A}}^B \} = \frac{1}{2 \sqrt{2\pi} \sigma_{c_2}}\int_{-R}^R \exp{\left( -\frac{(c_2 - \mu_{c_2})^2}{2\sigma_{c_2}^2} \right)}\\
    &\left[\!\text{erf} \left( \frac{\sqrt{R^2\!-\!c_2^2}\!-\!\mu_{c_1}}{\sigma_{c_1} \sqrt{2}} \right)\!+\!\text{erf} \left( \frac{\sqrt{R^2\!-\!c_2^2}\!+\!\mu_{c_1}}{\sigma_{c_1} \sqrt{2}} \right)\!\right] \text{d}c_2,
\end{split}
\end{equation} where $\text{erf}(x)$ denotes the error function of $x$. Note that (\ref{own_2}) is only computable by numerical integration. 
\begin{remark}(Marginalization of the Radius)
The marginalization of (\ref{own_1}) over the radius can be found in \cite{AP91} p.190. Moreover, for zero-mean Gaussians and also isotropic uncertainty (i.e., $\sigma_{c_1} = \sigma_{c_2}$), further derivation steps are provided.
\end{remark}

\subsection{Approach 3: Global Coordinates}\label{sec_global}

The case of global coordinates, also known as the offset circle, follows directly from Section \ref{sec_local_coordinates}, and has been published by \cite{didonato1961integration}.
We find the derivation by offsetting the collision circle (red circle in Figure \ref{local_circles}) by the global ego vehicle position, i.e., $(\hat{c}_{1, e}, \hat{c}_{2, e})$. Naturally, that introduces the transformations $c_1 = \hat{c}_1 - \hat{c}_{1,e}, c_2 = \hat{c}_2 - \hat{c}_{2,e}$, where $(\hat{c}_1, \hat{c}_2)$ represent the global random object position, shifting the random vector $\boldsymbol{q}_o$. Hence, we obtain $\hat{\tilde{\mathcal{A}}}^B := \{ \hat{\boldsymbol{q}} \in \mathbb{R}^2 \mid (\hat{c}_1 - \hat{c}_{1,e})^2 + (\hat{c}_2 - \hat{c}_{2,e})^2 \leq R^2 \}$ as the collision set. 
Substituting the transformation into (\ref{own_1}) and applying the same derivation steps as in Section \ref{sec_local_coordinates}, recalling that $\mu_{c_1} = \hat{c}_{1,o} - \hat{c}_{1,e}, \mu_{c_2} = \hat{c}_{2,o} - \hat{c}_{2,e}$, gives
\begin{equation}\label{eq_global}
\begin{split}
        & \mathbb{P}\{\hat{\boldsymbol{q}}_o \in \hat{\tilde{\mathcal{A}}}^B \}=\\
        &\frac{1}{2 \sqrt{2\pi} \sigma_{c_2}}\int_{\hat{c}_{2,e}-R}^{\hat{c}_{2,e} + R} \exp{\left( -\frac{((\hat{c}_2 - \hat{c}_{2,e}) - \hat{c}_{2,o})^2}{2\sigma_{c_2}^2} \right)}\\
        &\Bigg[ \text{erf} \left( \frac{\hat{c}_{1,e} -\hat{c}_{1,o}+\sqrt{R^2 - (\hat{c}_2 - \hat{c}_{2,e})^2}}{\sigma_{c_1} \sqrt{2}} \right)  \\
        &+ \text{erf} \left( \frac{\hat{c}_{1,o} -\hat{c}_{1,e}+\sqrt{R^2 - (\hat{c}_2 - \hat{c}_{2,e})^2}}{\sigma_{c_1} \sqrt{2}} \right) \Bigg] \text{d}\hat{c}_2.
\end{split}
\end{equation}
\subsection{Approach 4: Polar Coordinates}\label{sec_polar}
Since the shape approximations are circular, we can transform the problem to polar coordinates. With the polar transformation, the integral bounds are not a function of each other (compare to (\ref{own_1})). Let us define $\rho$ and $\phi$ such that $c_1 = \rho \cos{(\phi)}$, $c_2 = \rho \sin{(\phi)}$, $\mu_{c_1} = d_{\mu} \cos{(\mu_\phi)}$, $\mu_{c_2} = d_{\mu} \sin{(\mu_\phi)}$, where $d_{\mu} = \sqrt{\mu_{c_1}^2 + \mu_{c_1}^2}$ denotes the mean distance between both circles and $\mu_\phi = \text{atan2}(\mu_{c_1}, \mu_{c_2})$ is the mean angle. With this transformation, we obtain
\begin{equation}\label{eq_polar}
\begin{split}
    &p_{\boldsymbol{q}_o}(\phi, \rho) = \frac{1}{2 \pi \sigma_{c_1}\sigma_{c_2}} \text{exp} \Bigg[ \frac{(\rho \cos{(\phi)} - d_{\mu} \cos{(\mu_\phi))^2}}{-2\sigma_{c_1}^2} \\
    & -\frac{(\rho \sin{(\phi)} -  d_{\mu} \sin{(\mu_\phi))^2}}{2\sigma_{c_2}^2}  \Bigg].
\end{split}
\end{equation}
In order to solve for the POC, we integrate as 
\begin{equation}\label{double}
\begin{split}
\mathbb{P}\{(\phi_o, \rho_o)\!\in\![0, 2\pi)\!\times\![0,\!R]\}\!=\!
\int_0^R\!\int_0^{2\pi}\! \rho \, p_{\boldsymbol{q}_o}\!(\phi,\!\rho)\text{d}\phi \text{d} \rho.
\end{split}
\end{equation}
\begin{remark}(Marginals of Polar Coordinates)
    The marginalization of (\ref{eq_polar}) can be found in \cite{cooper2020toolbox}, where many other variants of the bivariate Gaussian are also discussed.
\end{remark}

\subsection{Comparison of Approaches}\label{sec_comp}

To find the computationally most efficient approach, we implement all approaches in numerical software. The MCS approach (\ref{eq_MCS}) is implemented in an algorithm as presented by \cite{lambert2008collision}, Algorithm 1. Hereto, we use $N_J = 10^4$ samples, which yields a three digits precision for the POC. The approaches 1 - 4 require numerical integration which we perform with the global adaptive quadrature method, for which the tolerances are set to also yield three digits precision. To highlight the benefit of the proposed derivation for the local and global coordinate approach, we integrate the initial double integral and the single integral. For each method, including MCS, we repeated the calculation $10^4$ times and include the time for coordinate transformations. The average computational time needed for one computation is used as a comparative metric. We used a notebook with an Intel i7-9850H processor and \qty{32}{GB} of memory as a computational platform. The results are presented in Table \ref{table_computing}.
\begin{table}[hb]
\begin{center}
\caption{Results: Average Computing Times}\label{table_computing}
\begin{tabular}{lc}
Approach & Average Computing Time [ms]\\\hline
1) Monte Carlo sampling (\ref{eq_MCS})&  56.340 \\ 
2) Local coordinates (\ref{own_1}) & 0.275\\ 
2) Local coordinates (\ref{own_2}) & 0.171\\ 
3) Global coordinates & 0.372 \\ 
3) Global coordinates (\ref{eq_global})& 0.208 \\ 
4) Polar coordinates (\ref{eq_polar})& 0.395 \\ \hline
\end{tabular}
\end{center}
\end{table}
Evidently, the local coordinate approach poses the least computational effort, while the MCS approach takes the most time to compute. Overall, the MCS approach is two orders of magnitude slower than any other approach. We note that the difference between the local and the global approach is approx. \qty{0.1}{\ms}. A more significant reduction is found between the single and the double integral of the same approach. Here, the presented derivations yield a significant computational benefit. The polar approach takes the most time out of the numerical integration methods. We conjecture that this could be improved by reducing (\ref{eq_polar}) to a single integral, however, we did not find a solution for anisotropic uncertainty, i.e., $\sigma_{c_1} \neq \sigma_{c_2}$.

\section{Multi-Circle-to-Circle Collision} \label{sec_multi_circle}

Having determined the approach with the least computational effort, we proceed to apply that approach to the multi-circle-to-circle collision case, where the two-circle-to-circle case is depicted in Figure \ref{fig_problem_statement}. We use that case to derive the geometric conditions for ego vehicle approximations of multiple circles along multiple axis. The problem of multi-circle approximations is essentially different to that of Section \ref{sec_circle_to_circle}, since there are object positions for which the object's circle collides with \textit{multiple} ego vehicle circles.

\subsection{Single-Axis Circle Placement}\label{sec_single_axis}

Given the fact that the ego's circles are placed along one axis, the object circle can collide with two ego circles when $d_c \leq 2R$, which is clearly the case in Figure \ref{fig_problem_statement}, and is always guaranteed by (\ref{eq_placing}), since $w, r_o \in \mathbb{R}_{>0}$. The set of center positions for which the object's circle collides with two ego circles is lens-shaped (see the red lens in Figure \ref{fig_lens}). That set (i.e., the lens) is in fact accounted for twice by computing the POC of both ego's circles separately and adding the individual POCs, thereby violating Definition \ref{def011}. To retain a proper probability, we calculate the POC for both ego circles with the object circle and subtract the POC of the lens with the object circle. The properties of Definition \ref{def011} allows us to do so, which we show in the following.
\begin{lemma}\label{lem_01}
The POC of the two-circle-to-circle case with $d_c$ and $r_e$ identified by (\ref{eq_placing}) is a probability given by \begin{equation}\label{eq_two_cirles}
\begin{split}
    &\mathbb{P}\{ \boldsymbol{q}_o \in \tilde{\mathcal{A}}_{coll}\} = \\
    &\mathbb{P}\{ \boldsymbol{q}_o \in \tilde{\mathcal{A}}_{1}^{B}\}  +\mathbb{P}\{ \boldsymbol{q}_o \in \tilde{\mathcal{A}}_{2}^{B}\} - \mathbb{P}\{ \boldsymbol{q}_o \in \tilde{\mathcal{A}}_{1}^{L}\},
\end{split}
\end{equation}
where $\{ \boldsymbol{q}_o \in \tilde{\mathcal{A}}_{1}^{B} \}$ represents the event of the front ego circle colliding with the object's circle, $\{ \boldsymbol{q}_o \in \tilde{\mathcal{A}}_{2}^{B} \}$ represents the event of the ego's rear circle colliding with the object, and $\{ \boldsymbol{q}_o \in \tilde{\mathcal{A}}_{1}^{L} \}$ refers to the event of the ego's lens colliding with the object's circle.
\end{lemma}
\begin{proof} By Definition \ref{def_random_vector} we have $\mathbb{P}\{ \boldsymbol{q}_o \! \in \!\tilde{\mathcal{A}}_n\} \! = \! \mathbb{P}\{ \boldsymbol{q}_o(\omega) \!\in\! \Omega \mid \boldsymbol{q}_o(\omega)\! \in\! \tilde{\mathcal{A}}_n\}$, where the set $\{ \boldsymbol{q}_o(\omega) \!\in \!\Omega \mid \boldsymbol{q}_o(\omega)\! \in \!\tilde{\mathcal{A}}_n\}\! =\! \mathcal{A}_n$.
Let $\mathcal{A}^L_1\!=\!\mathcal{A}_{1}^{B} \cap \mathcal{A}_{2}^{B}$, the sets $\mathcal{A}_{1}^{B}$ and $\mathcal{A}_{2}^{B} \backslash (\mathcal{A}_{1}^{B} \cap \mathcal{A}_{2}^{B}) $ are disjoint, so $\mathcal{A}_{coll}\!=\!\mathcal{A}_{1}^{B} \cup \mathcal{A}_{2}^{B} \backslash (\mathcal{A}_{1}^{B} \cap \mathcal{A}_{2}^{B})$. From Definition \ref{def011}(ii) it follows that
\begin{equation}\label{eq_p1}
    \mathbb{P}\{ \mathcal{A}_{1}^{B} \cup \mathcal{A}_{2}^{B}\} =
    \mathbb{P}\{\mathcal{A}_{1}^{B} \} + \mathbb{P}\{ \mathcal{A}_{2}^{B} \backslash (\mathcal{A}_{1}^{B} \cap \mathcal{A}_{2}^{B}) \} .
\end{equation}
Furthermore, the sets $\mathcal{A}_{1}^{B} \cap \mathcal{A}_{2}^{B}$ and $\mathcal{A}_{2}^{B} \backslash (\mathcal{A}_{1}^{B} \cap \mathcal{A}_{2}^{B})$ are disjoint, where $\mathcal{A}_{2}^{B} = (\mathcal{A}_{1}^{B} \cap \mathcal{A}_{2}^{B}) \cup \mathcal{A}_{2}^{B} \backslash (\mathcal{A}_{1}^{B} \cap \mathcal{A}_{2}^{B})$. Again, utilizing Definition \ref{def011}(ii) we find
\begin{equation}\label{eq_p2}
    \mathbb{P}\{ \mathcal{A}_{2}^{B}\} =
    \mathbb{P}\{\mathcal{A}_{1}^{B} \cap \mathcal{A}_{2}^{B} \} + \mathbb{P}\{ \mathcal{A}_{2}^{B} \backslash (\mathcal{A}_{1}^{B} \cap \mathcal{A}_{2}^{B}) \}.
\end{equation}
With $\mathbb{P}\{ \mathcal{A}_{1}^{L}\} = \mathbb{P}\{\mathcal{A}_{1}^{B} \cap \mathcal{A}_{2}^{B} \}$ and substituting (\ref{eq_p2}) into (\ref{eq_p1}) we obtain 
\begin{equation}\label{eq_proof}
        \mathbb{P}\{ \mathcal{A}_{1}^{B} \cup \mathcal{A}_{2}^{B}\} =
         \mathbb{P}\{ \mathcal{A}_{1}^{B}\} + \mathbb{P}\{ \mathcal{A}_{2}^{B}\} - \mathbb{P}\{ \mathcal{A}_{1}^{L}\}.
\end{equation}
Therefore, by Definition \ref{def_random_vector} and $\mathcal{A}_{1}^{B} \cup \mathcal{A}_{2}^{B} = \mathcal{A}_{coll}$, (\ref{eq_two_cirles}) is retrieved.
\end{proof}
Suppose the ego vehicle is approximated with $N_c$ circles placed along ${c}_1$, sized and placed according to (\ref{eq_placing}), than the POC is given by 
\begin{equation}\label{eq_single_axis}
    \mathbb{P}\{ \mathcal{A}_{coll}\} = \sum^{N_c}_{i = 1}\mathbb{P}\{ \mathcal{A}_i^B\} - \sum^{N_l}_{j = 1}\mathbb{P}\{ \mathcal{A}_j^L\},
\end{equation}
 where $N_l = N_c - 1$, i.e., the number of lenses. Further  $\mathcal{A}_i^B$ denotes collisions with the individual circles $\mathcal{A}_j^L$ and denotes collisions with the individual lenses. The proof of the expression above follows the same argument as the proof of Lemma \ref{lem_01}. 
\begin{figure}[t]
\begin{center}
\includegraphics[width=8.4cm]{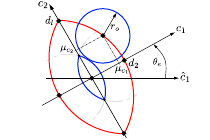}  % The printed column width is 8.4 cm.
\caption{Geometry of the lens, the red line denotes all object circle positions leading to collisions with both ego circles.}
\label{fig_lens}
\end{center}
\end{figure}
\subsection{Multi-Axis Placement}\label{sec_multi_axis_placement}
\begin{figure}[t]
\begin{center}
\includegraphics[width=8.4cm]{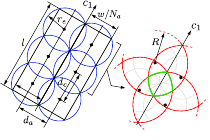}  % The printed column width is 8.4 cm.
\caption{Exemplary coverage (left) and quadruple collision lens overlap, green area in right depiction. For better visibility, the blue circles are shaded in gray in the right depiction.}
\label{fig_multi-circle}
\end{center}
\end{figure}
For this case we require that $N_a$ axes are parallel to the $c_1$ axis and divide the rectangle in smaller sub-rectangles of length $l$ and width $w/N_a$, placed at a distance $d_a = w/N_a$, that ensures $d_a \leq 2R$ (see (\ref{eq_placing}) and Figure \ref{fig_multi-circle} for an exemplary case of two axes and six circles). For the multi-axis case, $N_c$ must satisfy $\mod(N_c, N_a) = 0$. While the single-axis case yielded upright lenses along the axis, representing object positions where it is to collide with two circles, the multi-axis case introduces horizontal lenses in-between the axes which must be accounted for. However, one cannot apply (\ref{eq_single_axis}), because the overlap of four collision lenses created by four adjacent circles would cancel out all points within that overlap area (see green area in Figure \ref{fig_multi-circle}). Therefore one must add the POC of the overlap area, which is shaped by four arcs of radius $R$. We find that the POC for the multi-axis case is given by
\begin{equation}\label{eq_multi_axis}
        \mathbb{P}\{ \mathcal{A}_{coll}\} = \sum^{N_c}_{i = 1}\mathbb{P}\{ \mathcal{A}_i^B\} - \sum^{N_l}_{j = 1}\mathbb{P}\{ \mathcal{A}_j^L\} + \sum^{N_o}_{m = 1} \mathbb{P}\{ \mathcal{A}_m^O \},
\end{equation}
with the total number of lenses $N_l$ and overlapping lenses $N_o$ as
\begin{equation}
    N_l = N_c \left( 2 - \frac{1}{N_a}\right) - N_a, \;\; N_o = \left(\frac{N_c}{N_a} - 1 \right) (N_a -1),
\end{equation}
and $\mathcal{A}_{m}^O$ denoting collisions with the individual lens overlap areas. Similar to (\ref{eq_single_axis}), the proof follows the same argument of the proof of Lemma \ref{lem_01}. Therefore, the POC is obtained with (\ref{eq_multi_axis}) for any number of axis and $N_c$ number of circles, if the geometric conditions from (\ref{eq_placing}) and Section \ref{sec_multi_axis_placement} are satisfied. 
\subsection{POC of the Lens}
Similar to the circle-to-circle case, we have to define the collision line (see the red lens in Figure \ref{fig_lens}), i.e., all object positions for which a collision with both ego circles occurs. This yields the coordinate $(d_2,0)$, the upper integral limit of the collision lens on the $c_1$ axis, where $d_2$ is given by
\begin{equation}\label{eq_d_2}
    d_2 = \sqrt{(r_e + r_o)^2-\frac{d_c^2}{4}}.
\end{equation}
We derive the upper integral limit on the $c_2$ axis as a function of the outer integral, that is,
\begin{equation}\label{eq_dl}
    d_l(c_2) = \sqrt{(r_e + r_o)^2 - c_2^2} -\frac{d_c}{2}.
\end{equation} 
Hence, we find the integral bounds as $c_{1} \in [-d_2,d_2]$ and $c_{2} \in [-d_l(c_{1}),d_l(c_{1})]$. 
We use (\ref{eq_d_2}), (\ref{eq_dl}) and apply the same steps as in (\ref{own_1}) - (\ref{own_2}) to obtain for the POC of the lens with the object circle:
\begin{equation}\label{eq_lens}
    \begin{split}
    &\mathbb{P}\{ \boldsymbol{q}_o \in \tilde{\mathcal{A}}^{L}\} = 
     \frac{1}{2 \sqrt{2\pi} {\sigma}_{c_2}}\int_{-d_2}^{d_2} \exp{\left( -\frac{({c}_1 - {\mu}_{c_2})^2}{2{\sigma}_{c_2}^2} \right)}\\
    &\Bigg[ \text{erf} \left( \frac{\sqrt{(r_e + r_o)^2 - {c}_1^2} -\frac{d_c}{2}  - {\mu}_{c_1}}{{\sigma}_{c_1} \sqrt{2}} \right) + \\
    &\text{erf} \left( \frac{\sqrt{(r_e + r_o)^2 - {c}_1^2} -\frac{d_c}{2}  + {\mu}_{c_1}}{{\sigma}_{c_1} \sqrt{2}} \right)\Bigg] \text{d}{c}_1.
\end{split}
\end{equation}
With (\ref{eq_lens}) and (\ref{own_2}) we can compute the POC for the two-circle-to-circle collision by applying (\ref{eq_two_cirles}) and placing the circles as described with (\ref{eq_placing}) along one axis.

\section{Error Bounding}\label{sec_error}

With the approach in the previous sections, we always over-approximate the ego vehicle's shape with (\ref{eq_placing}). That is, the ego vehicle is always entirely covered by the circles. Consequently, depending on the length and width of the rectangle and the number of circles and their placement, we are also always over-approximating the POC by some amount. Given the presented case of two-circle-to-one-circle collisions, (\ref{eq_two_cirles}) gives an upper-bound of the POC, denoted $\mathbb{P}\{ \mathcal{A}_{coll}^{up}\}$. To derive a lower-bound, we parameterize circles not according to (\ref{eq_placing}). Instead, we place them such that they are always fully encapsulated by the rectangle. The parameterization for two circles is straightforward and given by
\begin{equation}\label{eq_underapp}
    r_e = \frac{w}{2}, \; \; \; d_c = l - w.
\end{equation}
With (\ref{eq_underapp}), we can apply (\ref{eq_lens}) and (\ref{own_2}) and (\ref{eq_two_cirles}), from which we obtain the lower bound of the POC $\mathbb{P}\{ \mathcal{A}_{coll}^{low}\}$.
We measure the bounding corridor with the difference:
\begin{equation}
    \Delta_a = \mathbb{P}\{ \mathcal{A}_{coll}^{up}\} - \mathbb{P}\{ \mathcal{A}_{coll}^{low}\}.
\end{equation}
We demonstrate the resulting approach in the following section.

\section{Simulation Example}\label{sec_example}
\begin{figure}[t!]
\begin{center}
\includegraphics[width=8.4cm]{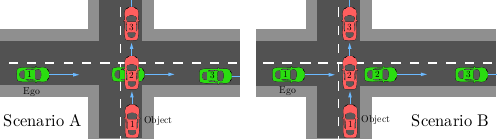}  % The printed column width is 8.4 cm.
\caption{Left: Scenario A, right: Scenario B. The green vehicle represents the ego vehicle and the red vehicle the object. The actors are displayed at three different time steps: 1: $t=$ \qty{0}{s}, 2: $t=$ \qty{4}{s}, 3: $t=$ \qty{8}{s}.}
\label{fig_simulation_scenario}
\end{center}
\end{figure}

To demonstrate the proposed approach, we choose a representative intersection scenario as a simulation example. Herein, we let the ego vehicle and object travel at constant velocity and along a straight trajectory. Both trajectories are perpendicular to each other; hence they are intersecting. We model two scenarios, A and B, respectively, both scenarios are depicted in Figure \ref{fig_simulation_scenario}. The ego vehicle is modeled by two circles placed in accordance with (\ref{eq_placing}) and the object by a single circle. Scenario A: Both actors collide; since we do not model the collision itself, they drive through each other. This means that the POC must be approaching one during the collision period if the standard deviations are sufficiently small. Scenario B: The object passes the intersection before the ego vehicle. Hence, the vehicles do not collide. In addition to the upper and lower bound (see Section \ref{sec_error}), we approximate the actual POC by MCS of the ego rectangle colliding with the circular object. We model the uncertainty depending on the distance between both actors to simulate that the measurement and estimation of closer objects are generally more accurate. We choose a logistic function to model the relative distance-dependent standard deviations as:
\begin{equation*}
    \Sigma(d) = \frac{1}{1+ \exp{[-\lambda (d - d_0)]}}\Sigma_{max},
\end{equation*}
where $\lambda, d_0$ being free parameters, $\Sigma_{max} \in \mathbb{R}^{2\times2}$ is diagonal with positive entries representing the maximum standard deviations and $d=\|\boldsymbol{q}_{e}-\boldsymbol{q}_o\|$. The parameters are provided in the Appendix. The results are displayed in Figure \ref{fig_results}. In scenario A, one can see that the POC is approaching one for the duration of the collision. With $k \in \mathbb{N}$ denoting discrete-time, we find that the bounding corridor $\Delta_{a,k} \leq 0.08$ for all $k$. Regarding scenario B, the maximum POC is less than $0.40$, when the distance between the ego vehicle and object is \qty{2.77}{m}. The maximum bounding difference is lower than in scenario A and for that we find $\Delta_{a,k} \leq 0.07$ for all $k$. As expected, the POC of the ego rectangle colliding with the object circle is always within the corridor of the upper and lower bound. The fluctuations of the MCS result are explained by the sampling method and vanish when the number of samples increases. From Figure \ref{fig_results} we conclude that the proposed approach allows for accurate estimation of the real POC. We note, that the error bounds can further be tightened by taking more circles in the approximation.
\begin{figure}[t!]
\begin{center}
\includegraphics[width=8.4cm]{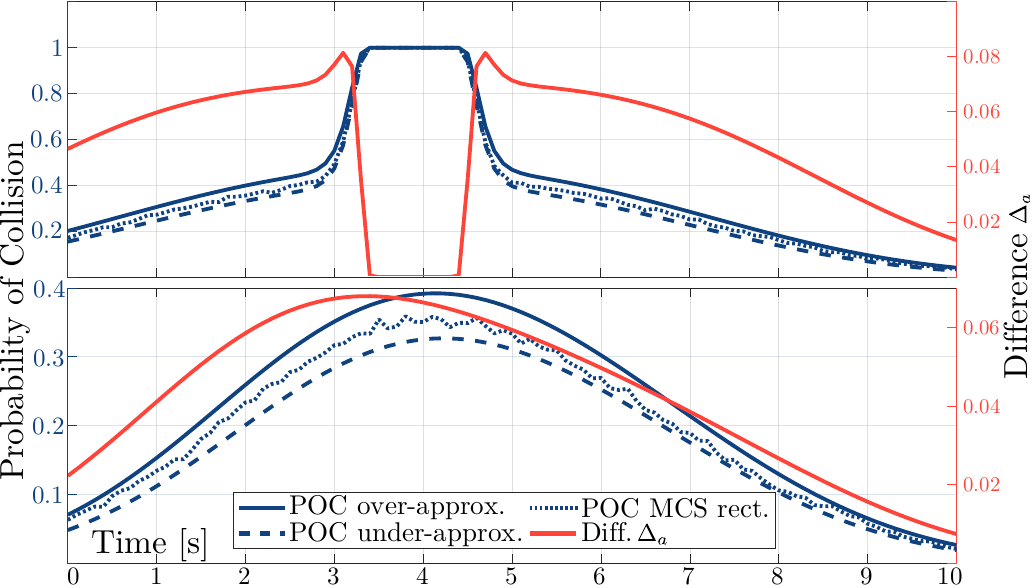}  % The printed column width is 8.4 cm.
\caption{Results of the simulation example. The upper plot represents scenario A, the lower plot represents scenario B.}
\label{fig_results}
\end{center}
\end{figure}
\section{Discussion of the Results} \label{sec_discussion}
The article is geared to address practitioners, to whom two questions may come into mind: can the POC be computed fast enough for my application? And, is the approximation good enough for my application? We note that the computational efficiency can be further improved by a) parallelization, since the POCs for each circle and lens can be computed independent of the other circles and lenses, and b) a more efficient low-level implementation. On this account, we believe that the approximation error and computational effort are sufficiently small for most motion planning applications. Regarding the error, one can see in Figure \ref{fig_results} that the error reduces (i.e., the bounding corridor narrows) when the POC reduces. In practice, only small POCs within a motion planner would be allowed. For these small POCs, the error will also be sufficiently small to a point where we argue that the error is negligible. 

\section{Conclusions and Future Work}\label{sec_conclusion}
In this paper, we provide a method for over-approximating the probability of collision (POC) for circular shape approximations with Gaussian uncertainty, as the state-of-the-art method of Monte Carlo sampling is computationally expensive and does not guarantee an under or over-approximation for finite amounts of samples. Thereby, we address an open gap in the literature, as outlined in Section \ref{sec_introduction}. Given a Gaussian distribution, we derived various methods (see Section \ref{sec_circle_to_circle}) and compared these in Section \ref{sec_comp} based on computational effort. We find that the proposed local coordinate approach is computationally most efficient. In Section \ref{sec_multi_circle}, we extend the approach to multi-circle shape approximations with a multi-axis placement and guarantee to retain a probability when certain geometric conditions are satisfied. Having the multi-circle-to-circle derivation, which over-approximates the POC, we present an intuitive method to under-approximate the error in Section \ref{sec_error}. Finally, we demonstrate the resulting algorithm in a representative example (Section \ref{sec_example}). We find that the error remains sufficiently small for the tested scenarios, and the computational effort is low, allowing for scalability as needed by motion planning algorithms. However, it remains an open challenge to derive the POC for the multi-circle-to-multi-circle collision. Such appears challenging as the objects' orientations are generally uncertain if the position is assumed to be uncertain. If the object's heading angle becomes a random variable, one obtains a trivariate Gaussian, for which we expect the computational effort to increase. Additionally, the construction of the collision set, i.e., all object positions leading to a collision, becomes non-trivial since the collision depends on the heading. On that account, solving the rotating lens to rotating lens case presents the key. Future work will investigate orientation-depended multi-circular object approximations as well as the implementation of the presented work within a stochastic model predictive control-based motion planner. Here, the trade-off between estimation accuracy and computational effort by adding varying numbers of circles to each actor is particularly interesting. The integration inside a motion planner allows for testing in more realistic scenarios. 

\bibliography{lib}
\bibliographystyle{ieeetr}
\appendix \label{appendix} Simulation parameters: \textit{uncertainty}: $\lambda = 6, d_0 = 1, \Sigma_{max} =\begin{pmatrix} 2 & 0\\0 & 5\end{pmatrix}$. \textit{Shapes}: Ego dimensions $w = 2, l = 4.5$, number of ego circles $n_c = 2$, radius object $r_o=2$. Scenario A: \textit{Ego vehicle}: initial configuration $\boldsymbol{y}_{e,0} = (0, 4, 0)$, velocity $v_e = 1$ turn rate $\omega_e = 0$. \textit{Object}: initial configuration $\boldsymbol{y}_{o,0} = (4, 0, \pi/2)$, velocity $v_o = 1$ turn rate $\omega_o = 0$. Scenario B: \textit{Ego vehicle}: initial configuration $\boldsymbol{y}_{e,0} = (0, 4, 0)$, velocity $v_e = 1$ turn rate $\omega_e = 0$. \textit{Object}: initial configuration $\boldsymbol{y}_{o,0} = (6, 0, \pi/2)$, velocity $v_o = 1.5$ turn rate $\omega_o = 0$. 
\end{document}